\documentclass[conference]{IEEEtran}
\IEEEoverridecommandlockouts
\usepackage{cite}
\usepackage{amsmath,amssymb,amsfonts}
\usepackage{graphicx}
\usepackage{textcomp}
\usepackage{xcolor}
\usepackage{algorithm}
\usepackage{algpseudocode}
\usepackage{caption}
\usepackage{subcaption}
\usepackage{float}
\usepackage[english]{babel}
\usepackage{amsthm}
\usepackage{hyperref}

\newtheorem{theorem}{Theorem}
\newtheorem{assumption}{Assumption}

\algnewcommand\algorithmicforeach{\textbf{foreach}}
\algdef{S}[FOR]{ForEach}[1]{\algorithmicforeach\ #1\ \algorithmicdo}

\theoremstyle{definition}
\newtheorem{definition}{Definition}

\def\BibTeX{{\rm B\kern-.05em{\sc i\kern-.025em b}\kern-.08em
    T\kern-.1667em\lower.7ex\hbox{E}\kern-.125emX}}
\begin{document}

\title{A One-Shot Reparameterization Method for Reducing the Loss of Tile Pruning on DNNs \\
}

\author{\IEEEauthorblockN{Yanchen Li$^1$, Qingzhong Ai$^2$, and Fumihiko Ino$^1$}
\IEEEauthorblockA{1: Osaka University, 1-5 Yamadaoka, Suita, Osaka 565-0871, Japan, \{yc-li, ino\}@ist.osaka-u.ac.jp}
\IEEEauthorblockA{2: University of Electronic Science and Technology of China, Chengdu, China, qzai@std.uestc.edu.cn}
}

\maketitle

\begin{abstract}

Recently, tile pruning has been widely studied to accelerate the inference of deep neural networks (DNNs).
However, we found that the loss due to tile pruning, which can eliminate important elements together with unimportant elements, is large on trained DNNs.
In this study, we propose a one-shot reparameterization method, called \textit{TileTrans}, to reduce the loss of tile pruning.
Specifically, we repermute the rows or columns of the weight matrix such that the model architecture can be kept unchanged after reparameterization.
This repermutation realizes the reparameterization of the DNN model without any retraining.
The proposed reparameterization method combines important elements into the same tile; thus, preserving the important elements after the tile pruning.
Furthermore, TileTrans can be seamlessly integrated into existing tile pruning methods because it is a pre-processing method executed before pruning, which is orthogonal to most existing methods.
The experimental results demonstrate that our method is essential in reducing the loss of tile pruning on DNNs.
Specifically, the accuracy is improved by up to 17\% for AlexNet while 5\% for ResNet-34, where both models are pre-trained on ImageNet.

\end{abstract}


\section{Introduction}
\label{sec:intro}

Deep neural networks (DNNs) have demonstrated excellent performance on a wide variety of tasks, such as computer vision~\cite{xiong2021explore, liu2021swin}, speech recognition~\cite{popov2021grad, min2021meta}, and robotics~\cite{chebotar2021actionable, sunderhauf2018limits}. One backbone behind the growth in the abilities of DNNs is an increasing number of trainable parameters. For instance, a recent state-of-the-art model, CogView~\cite{ding2021cogview}, achieved superior performance on text-to-image generation tasks, thanks to its almost four billion parameters. However, training a model with many parameters is time- and resource-intensive. Therefore, the need for accelerating the computation on DNN is also increasing. Nowadays, several mainstream acceleration methods, such as pruning~\cite{han2015learning, wang2021accelerate, jordao2020discriminative}, mixed precision~\cite{micikevicius2017mixed}, and parallelization~\cite{dean2012large} have recently attracted significant attention. Among these methods, mixed precision and parallelization focus on the acceleration in the training phase of DNNs, whereas pruning concentrates on the inference time.

This study focuses on pruning for the increasing need to accelerate the application of DNNs. Pruning accelerates DNNs by deleting the trivial elements in the weight matrix. The importance of each element is evaluated by the \textit{importance score}~\cite{molchanov2019importance} so that the elements with small importance scores will be deleted. Furthermore, to quantify the result of pruning, we calculate the loss of pruning based on the importance score of the deleted elements, generally summation. As the goal of pruning, we hope to get a minor pruning loss, which brings higher inference accuracy, with a certain sparsity of pruning. 

However, there is a tradeoff between pruning loss and acceleration. On the one hand, pruning loss increases when pruning sparsity increases. On the other hand, acceleration from pruning is proportional to the sparsity. Consequently, various pruning methods~\cite{peng2021accelerating, ren2020darb, narang2017block, chen2021re} managed to better balance the tradeoff between the pruning loss and acceleration.

Among those pruning methods, tile pruning~\cite{guo2020accelerating} achieves a better tradeoff. Tile pruning deletes the elements in the weight matrix by tiles while other structured pruning methods eliminate the elements by a larger shape, such as a row. The advantages of deleting elements by tiles can be reflected in the following two aspects. First, owing to the smaller pruning unit of tiles, the pruning loss is relatively smaller than other structured pruning methods. Second, the computation for pruned tiles can be skipped with a small overhead, and thus the pruned DNNs are accelerated. However, we found that the trained DNNs empirically do not always fit well with tile pruning. Those important elements are usually eliminated with unimportant elements, leading to large pruning loss.

To this end, we propose a reparameterization method, called \textit{TileTrans}, based on the DNNs model for tile pruning. Instead of changing the pruning strategy, TileTrans reparameterizes the DNNs model before the pruning. Specifically, we use permutation to separate elements by their importance so that there are either important or unimportant elements in the same tile. In this way, we preserve more important elements after tile pruning and achieve minor pruning loss. TileTrans is user-friendly in the sense that we reparameterize the DNN model by a trick of matrix multiplication instead of retraining. Besides, we keep the model architecture unchanged and therefore introduce no extra computation for inference. Additionally, our algorithm is a straightforward heuristic algorithm, which means that the algorithm's efficiency is guaranteed. 

We evaluated TileTrans on two classical DNN models: AlexNet~\cite{krizhevsky2012imagenet} and ResNet-34~\cite{he2016deep}, where both models are first pretrained on the ImageNet dataset~\cite{deng2009imagenet}. We reparameterized the DNN models by TileTrans before tile pruning on the models. The experimental results show that TileTrans reduces the pruning loss and improves accuracy by 17\% on AlexNet and 5\% on ResNet-34, respectively.

The main contributions of our study are as follows:
\begin{itemize}
    \item We proposed a reparameterization method, named TileTrans, to permute the weight matrix such that important and unimportant elements are permuted in the same tile before tile pruning, respectively. As a result, we preserve more important elements after tile pruning, and thus we reduce the pruning loss.
    \item We reparameterized DNN models in a novel way based on a trick of matrix multiplication. Compared with other reparameterization methods, our method is user-friendly because TileTrans requires no extra retraining on DNN models. Furthermore, We kept the architecture of the model unchanged so that no additional computation occurs after TileTrans. 
    \item We proposed a heuristic method to build an appropriate permutation for TileTrans to avoid time-consuming searching for the optimal permutation. Besides, we verified that this heuristic method does reduce the loss of tile pruning experimentally.
\end{itemize}
Eventually, we deduce that TileTrans is capable of building DNN models that are more adaptable to tile pruning.

\section{Related work}
\label{sec:related}

Notably, TileTrans can be fused with any tile pruning method seamlessly.
Various pruning strategies~\cite{lin2019towards, guo2020accelerating, peng2021accelerating} have been proposed to reduce the loss of tile pruning. 
For example, Guo et al.~\cite{guo2020accelerating} deleted the tiles by the number of important elements in the tile. 
TileTrans reparameterizes DNN before the procedure of pruning; thus, it can be used together with these pruning strategies to improve the performance of tile pruning.

Compared with the major reparameterization methods on DNNs, TileTrans needs no extra training. 
Reparameterization has been used to change the data structure of DNN for many purposes~\cite{ding2019acnet, ding2021diverse, ding2021repmlp, ding2021repvgg}.
Specifically, reparameterization is also used to build a model with structural sparsity.
Most reparameterization methods design a particular loss function and train the model for a large number of iterations. 
For example, Wen et al.~\cite{wen2016learning} introduced the Lasso regression to the loss function and trained a model with structural sparsity;
Ding et al.\cite{ding2021resrep} proposed the ResRep, which punished unimportant channels in the weight matrices during training.
These studies achieved outstanding performance in building structural sparsity, but they required extra training on the models.
In contrast, our TileTrans only reparameterizes DNN by matrix multiplication without any training.

Some previous works also reduced the pruning loss by transforming the weight matrix.
However, these methods changed the model architectures; thus, increasing the cost of computation.
Liu et al.~\cite{liu2018frequency} added a transformation to each layer, by which they transformed the format of each weight matrix into the frequency domain. 
Guo et al.~\cite{guo2020accelerating} duplicated each weight matrix and saved the important elements in one of them.
These methods improved the accuracy of the pruned models but called for extra computation compared with the original architecture. 
In contrast, TileTrans transforms weight matrices without modifying the architecture.
In other words, we add no extra computation to the model.

\section{Preliminary}
\label{sec:pruning}

Before introducing TileTrans, we have to present the preliminary about tile pruning because we specialize TileTrans for reducing the loss of tile pruning.
Therefore, we will introduce the optimization objective and algorithm of tile pruning as follows.

\subsection{Importance Score}
\label{sec:importance_score}

The importance score is a criterion used to evaluate the importance of an element in the weight matrix. Given an element with weight $w$, we calculate the importance score as $\mathcal{E}(w)$, where $\mathcal{E}: \mathbb{R} \to \mathbb{R}$ denotes the criterion of importance score. There are many criteria in previous works, such as calculating the Hessian matrix~\cite{lecun1990optimal} and deriving from Taylor expansion~\cite{molchanov2019importance}.

We can quantitatively define the loss of pruning with a certain importance score. Generally, the pruning loss is an important score summation of the deleted elements. Given the weights of the model as $\mathbf{W}$ and targeted sparsity $s$, weights are pruned by  $\mathit{f}(\mathbf{W}, s ; \mathcal{E})$, where $f$ deletes the elements in $\mathbf{W}$ to the sparsity of $s$ depending on $\mathcal{E}$. Furthermore, pruning can be described as an optimization problem with the objective as follows:
\begin{equation}\label{eq:pruning_optimazation}
\begin{aligned}
     \min_{f} \quad & \mathcal{L}(\mathbf{W}, s; \mathcal{E}) = \left\vert \sum_w^{ w \in \mathbf{W} } \mathcal{E}(w) - \sum_w^{w \in \mathit{f}(\mathbf{W}, s; \mathcal{E})} \mathcal{E}(w) \right\vert \\
     \textrm{s.t.} \quad & \mathbf{W} \in \mathbb{W}, 0 \leq s \leq 1, 
\end{aligned}
\end{equation}
where $\mathcal{L}$ denotes the loss function of pruning and $\mathbb{W}$ is the set of all possible weights. In other words, the loss of pruning can be minimized by deleting the elements with the smallest importance scores. In tile pruning, there are certain restrictions on the shape of the pruning unit. Therefore, we evaluate the average importance score of the elements in the tiles for tile pruning and delete the most unimportant tiles. 

\subsection{Tile Shape}

The shape of tile pruning depends on the devices or libraries. 
For example, Chen et al.~\cite{chen2021re} pruned the weights in the shape of the loading unit for the tensor core so that the computation for pruned elements can be skipped with a small overhead.
Guo et al.~\cite{guo2020accelerating} used a shape of $128 \times 1$, because CUTLASS~\cite{nvidia2017cutlass} loads $128 \times 1$ elements at once from the global memory to the shared memory on the GPU system.
Given the tile shape of $a \times b$, tile pruning sometimes is called the ``block pruning''~\cite{peng2021accelerating} when $a=b$. 
Note that we regard the block pruning as tile pruning in this paper.
Especially, tile pruning in the shape of $1 \times 1$ is called unstructured pruning~\cite{liu2018rethinking}.
The pruning loss of unstructured pruning is usually lower than that of tile pruning, where $a > 1$ and $b > 1$. 
Unstructured pruning precisely deletes the unimportant elements while tile pruning eliminates the important elements together with the unimportant elements in the same tile. 
Therefore, unstructured pruning is usually used as the baseline for tile pruning when evaluating the accuracy of pruned models.

\begin{figure}[t]
    \centering
    \begin{subfigure}{.33\linewidth}
      \centering
      \includegraphics[width=\linewidth]{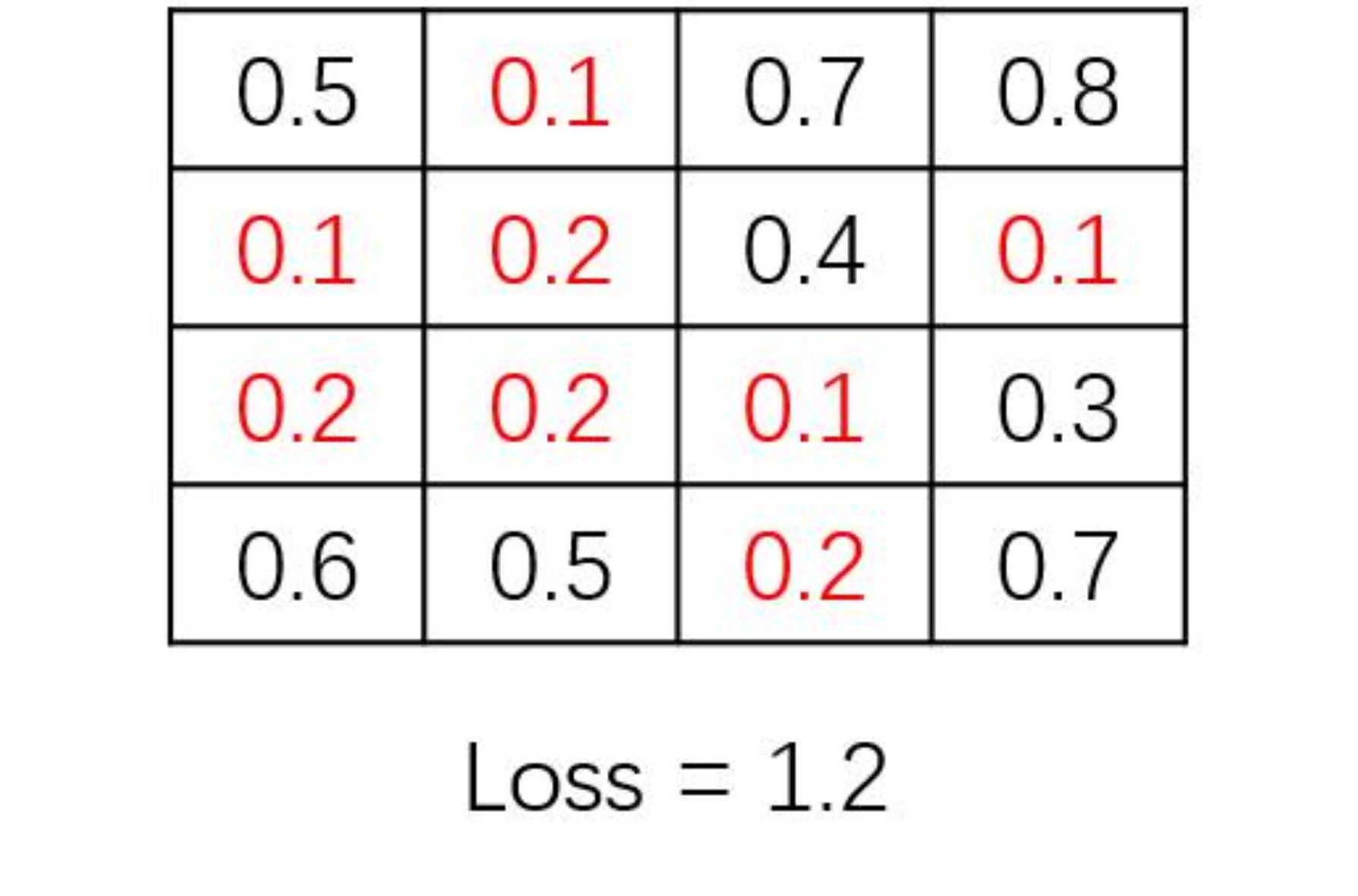}
      \caption{}
      \label{fig:sample_recon_baseline}
    \end{subfigure}%
    \begin{subfigure}{.33\linewidth}
      \centering
      \includegraphics[width=\linewidth]{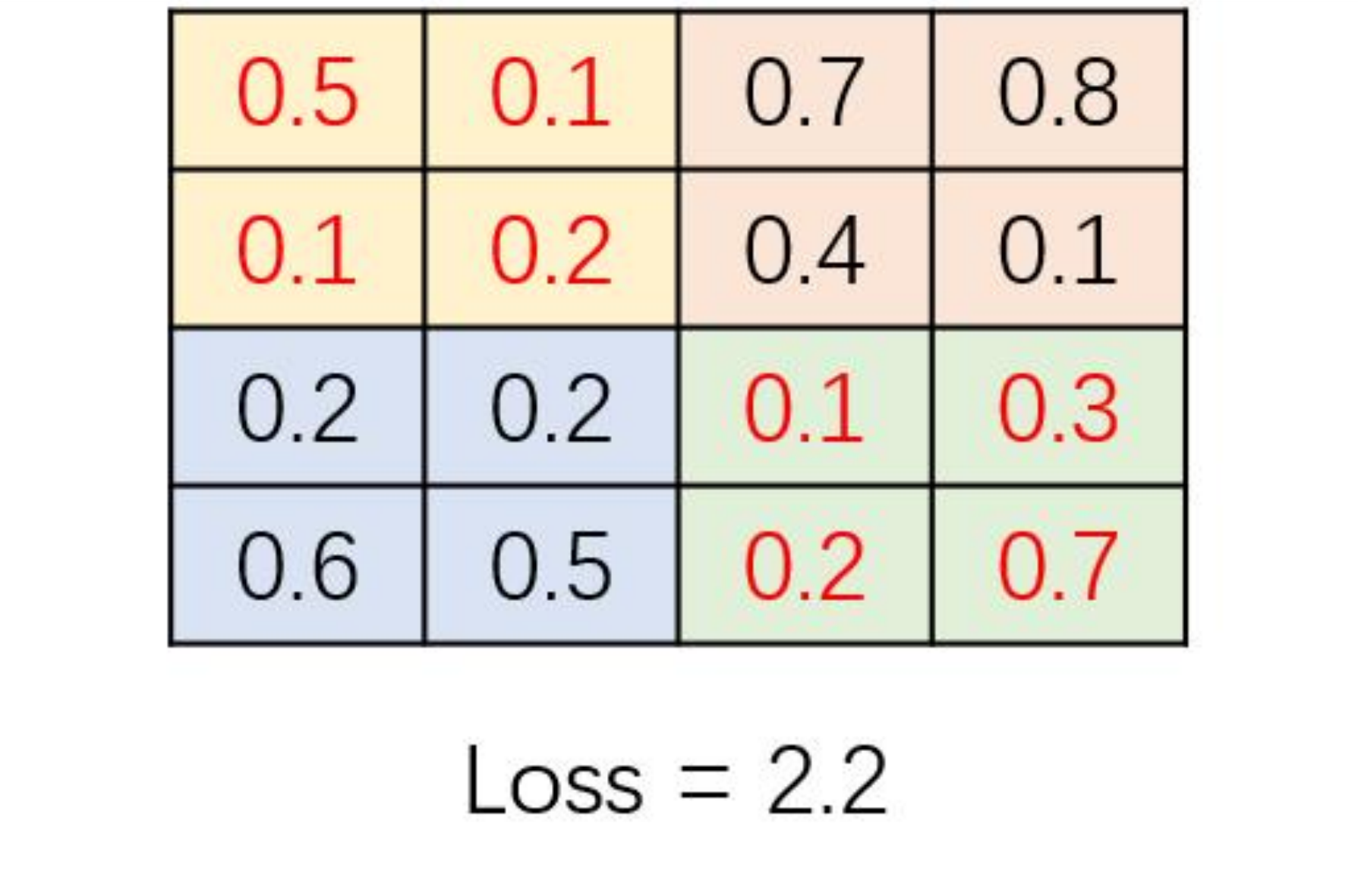}
      \caption{}
      \label{fig:sample_recon_before}
    \end{subfigure}%
    \begin{subfigure}{.33\linewidth}
      \centering
      \includegraphics[width=.9\linewidth]{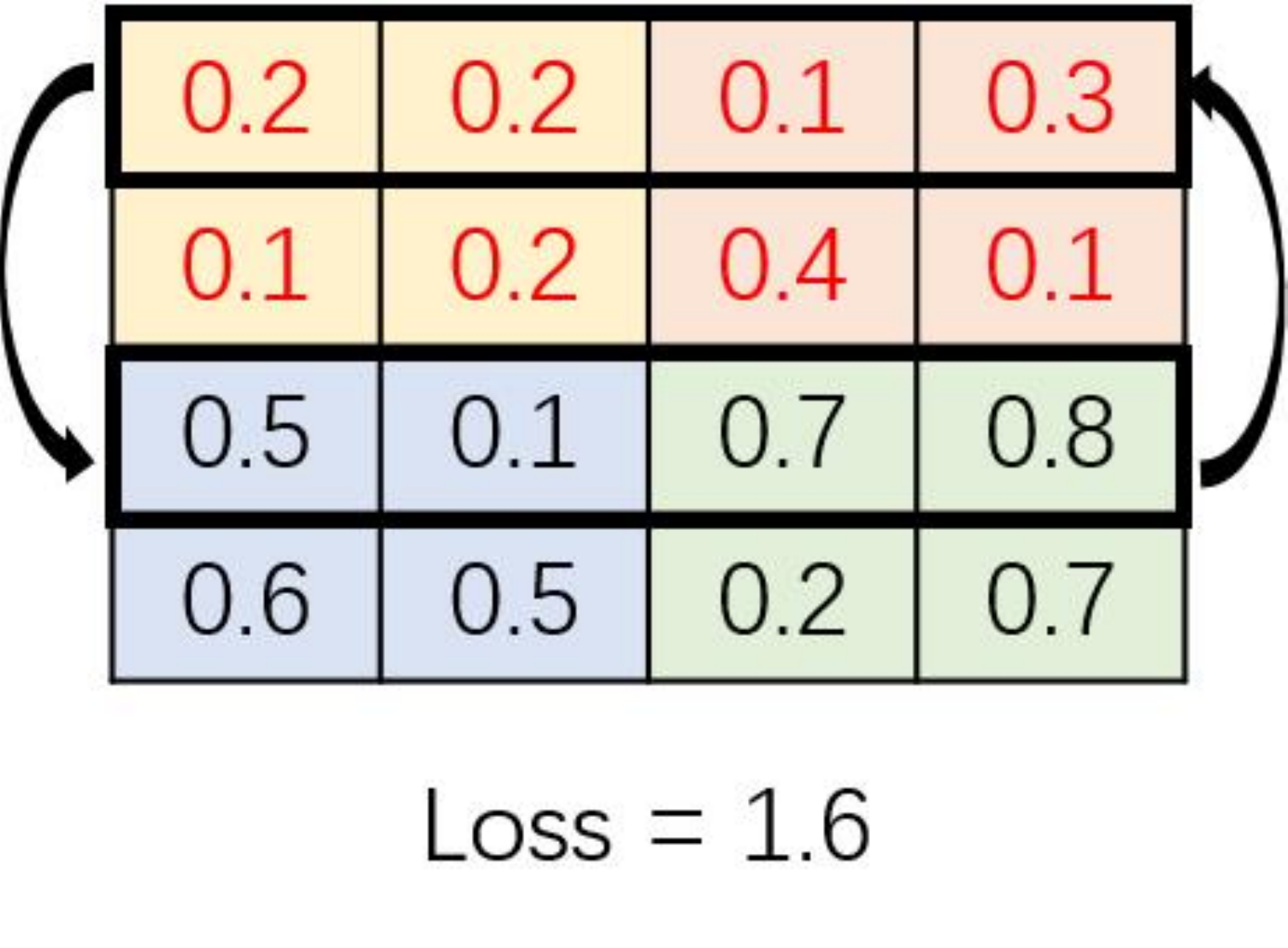}
      \caption{}
      \label{fig:sample_recon_after}
    \end{subfigure}
    \caption{Examples for pruning the $4 \times 4$ weight matrix to the sparsity of 0.5 with the unstructured pruning (a), the $2\times 2$ tile pruning before reparameterization (b), and the $2\times 2$ tile pruning after reparameterization (c), respectively. We fill the background with different colors for the corresponding tiles in (b) and (c) for easy observation. The numbers in cells are the importance scores for elements. The values of pruning loss are shown at the bottom of the figures. }
    \label{fig:sample_recon}
\end{figure}

\subsection{Algorithm of Tile Pruning}

For tile pruning, we evaluate the importance of every tile and delete unimportant tiles according to the importance score, which is calculated from the elements in the tile.
We show details in Algorithm~\ref{alg:tile_pruning}.
From lines 2 to 6, we traverse all weight matrices in the model and calculate the importance scores for every tile.
The importance score of a tile is defined as the average of the importance scores of the elements in that tile.
After evaluating for tiles, we set a threshold according to the sparsity at line 7.
Because the sparsity is the percentage of zero elements in all matrices, we delete unimportant tiles whose importance scores are smaller than the threshold.
Finally, we retrain the model to recover the accuracy for inference after tile pruning.

\begin{algorithm}
    \caption{ TilePruning($\mathbf{W}$, $\mathcal{E}$, $a$, $b$, $s$) }\label{alg:tile_pruning}
    \begin{algorithmic}[1]
        \Require Weight matrices of the model $\mathbf{W}$, metrics for importance score $\mathcal{E}$, height of tile $a$, width of tile $b$, and targeted sparsity $s$.
        \Ensure Pruned weight matrices $\mathbf{W}$.
        
        \State $\mathbf{V} \gets \emptyset$
        \ForEach{$W \in \mathbf{W}$ } \Comment{ Traverse the weights in the layers }
            \ForEach{$T \in W$}
                \State $\mathbf{V} \gets \mathbf{V} \cup \{ v~|~v = \sum_{i}^{a}\sum_{j}^{b}\mathcal{E}(T_{i,j})/ab  \} $
            \EndFor
        \EndFor
        \State Set threshold $t$ such that $|\{ v~|~v \in \mathbf{V}, v < t \}| / |\mathbf{V}| = s$
        \State Delete the tile with importance score smaller than $t$.
    \end{algorithmic}
\end{algorithm}

\section{Weights Transformation}
\label{sec:recon}

Naturally, the distribution of the elements in the weight matrices is usually unfriendly to tile pruning.
The important and unimportant elements are mixed in the same tile.
We show the examples in Figs.~\ref{fig:sample_recon_baseline} and \ref{fig:sample_recon_before}. 
For unstructured pruning in  Fig.~\ref{fig:sample_recon_baseline}, unimportant elements are precisely deleted, and thus the loss for unstructured pruning is the smallest.
For tile pruning in Fig.~\ref{fig:sample_recon_before}, the elements that are alive in unstructured pruning are deleted.
Consequently, the loss for tile pruning is more significant than for unstructured pruning.
To separate the important and unimportant elements, we propose TileTrans that reparameterizes DNN before the pruning by matrix transformation. 

\subsection{Idea of Permutation}
\label{sec:recon_idea}

The loss for tile pruning can be reduced by matrix transformation.
We transform the weight matrices to a better format, where the important elements are combined together in the same tiles as far as possible. 
For example, in Fig.~\ref{fig:sample_recon_after}, we switch the first row of the matrix with the third row.
Accordingly, more important elements are permuted in the same tiles than in Fig.~\ref{fig:sample_recon_before}.
Finally, the loss for the $2 \times 2$ tile pruning on the reparameterized weight matrix becomes smaller than that on the original weight matrix.

The critical point is that the transformation must keep the model's output unchanged.
The output of DNN is changed if we transform the weight matrix.
One option to recover the output is adding extra calculation after the inference.
However, we should avoid increasing the computational amount for inference; thus, we keep the model architecture unchanged.
Without modifying the architecture, we managed to recover the model's output by a trick of matrix multiplication. 

\subsection{ Algorithm of TileTrans}
\label{sec:recon_principle}
We designed a reparameterization method based on matrix multiplication. 
The main components of the DNN models are the linear layer~\cite{liu2017survey} and the convolutional layer~\cite{albawi2017understanding}, which can be represented as matrix multiplication.
Furthermore, the DNN model consists of multiple layers, where the calculations are transitive.
For example, Fig.~\ref{fig:partial_resnet} shows a partial architecture of ResNet, where the output of a layer is the input of another layer.
Utilizing this transitivity of the calculations in DNN, we can reparameterize DNN without any training by the properties of matrix multiplication.

\begin{figure}
    \centering
    \includegraphics[width=0.7\linewidth]{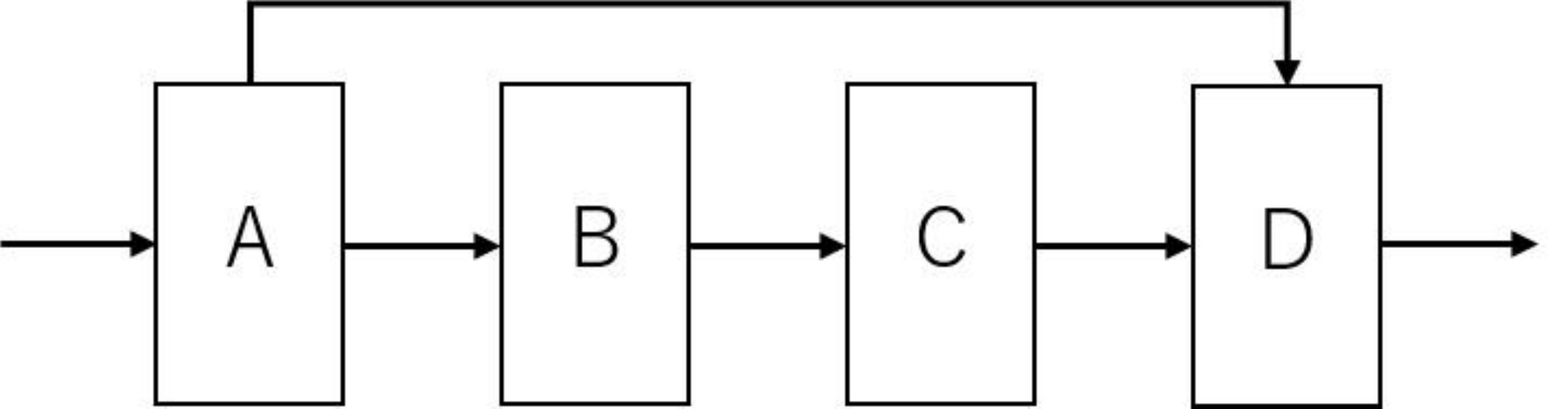}
    \caption{Partial architecture of ResNet. A block denotes a layer with weights. An arrow represents a data dependence inherent in the data flow.}
    \label{fig:partial_resnet}
\end{figure}

The reparameterization procedure is a series of operations that transform two layers.
For a single operation, we first build an invertible transformation $R$ for the previous layer and build its inverse $R^{-1}$ for the next layer.
In this way, the output of the transformed layer can be recovered by calculating the next layer. 
We take Fig.~\ref{fig:partial_resnet} as an example to show how this trick works.
Suppose that the weights of the layer $i$ is $W_i$, the input is $I_i$, and the output is $O_i$. The calculations in layer $A$ and $B$ then can be given as
\begin{equation}\label{eq:original_two_layers_0}
\begin{split}
    O_A = I_A W_A^T,
\end{split}
\end{equation}
\begin{equation}\label{eq:original_two_layers_1}
    O_B = O_A W_B^T.
\end{equation}
After that, we transform these weights by a matrix transformation $R$ as follows
\begin{equation}\label{eq:recon_two_layers_1}
\begin{split}
    O_A R^T = I_A (R W_A)^T,
\end{split}
\end{equation}
\begin{equation}\label{eq:recon_two_layers_2}
    O_B = O_A R^T (W_B R^{-1})^T.  
\end{equation}
In other words, we transform the previous layer and use the next layer to recover the output.
Therefore, the output of the two layers is unchanged, even though we transform their weights. 
Finally, we reparameterize the entire DNN model based on this trick.
However, the restraint for the reparameterization becomes more strict when the model architecture is more complex.

Some layers in the model have to share the same matrix transformation when there are residuals in the model. 
Residual~\cite{he2016deep} is a calculation that reuses the output of the previous layer as the input of the deeper layer. 
For example, in Fig.~\ref{fig:partial_resnet}, the output of layer A is the input of layers B and~D.
Before calculating the matrix multiplication of layer D, we first calculate the sum of $O_A$ and $O_C$ as follows:
\begin{equation}\label{eq:origin_partial_resnet_0}
    O_A = I_A W_A^T, 
\end{equation}
\begin{equation}\label{eq:origin_partial_resnet_1}
    O_B = O_A W_B^T, 
\end{equation}
\begin{equation}\label{eq:origin_partial_resnet_2}
    O_C = O_B W_C^T,
\end{equation}
\begin{equation}\label{eq:origin_partial_resnet_3}
    O_D = (O_A + O_C) W_D^T.
\end{equation}
Supposing the transformation for layer $i$ as $R_i$, the calculations for the transformed layers are given as
\begin{equation}\label{eq:recon_partial_resne_0}
    O_A R_A^T = I_A (R_A W_A)^T, 
\end{equation}
\begin{equation}\label{eq:recon_partial_resnet_1}
    O_B R_B^T = O_A R_A^T (R_B W_A R_A^{-1})^T, 
\end{equation}
\begin{equation}\label{eq:recon_partial_resnet_2}
    O_C R_C^T = O_B R_B^T (R_C W_C R_B^{-1})^T ,
\end{equation}
\begin{equation}\label{eq:recon_partial_resnet_3}
    O_D = (O_A R_A^T + O_C R_C^T) ( W_D R_C^{-1} )^T,
\end{equation}
where $O_D = (O_A + O_C) W_D^T$ only if $R_A = R_C$.
Thus, layers A and C must share the same transformation.
The example of Fig.~\ref{fig:partial_resnet} is simple, but the actual model is usually more sophisticated.

There may be groups of layers that must share the same transformation. For simplicity, we define the layers that have to share the same transformation as the \textit{layer group}. 
This layer group is built according to the relations of parents and children, which can be defined as follows.
\begin{definition}[Parents and Children]
Given a layer A, parents of A are the layers that send their output to A as the inputs of A. 
Meanwhile, children of A are the layers that receive the output of A as their inputs.
\end{definition}
\noindent
Specifically, we present a theorem as follows.
\begin{theorem}\label{theorm:layer_group}
Given a layer A in the DNN model, the parents of A must share the same transformation.
\end{theorem}
\begin{proof}
\noindent
According to Eq.~(\ref{eq:recon_partial_resnet_3}), assuming the weight is $W$, the transformation is $R$, and the output is $O$, if there are $n$ parents for a layer, the output of the layer is given by
\begin{equation}\label{eq:proof_layer_group_0}
    O = (O_1 R_1^T + O_2 R_2^T + ... + O_n R_n^T) (W R^{-1})^T,
\end{equation}
where $O_i$ denotes the output of the $i$-th parent, and $R_i$ denotes the transformation of the $i$-th parent. Then
\begin{equation}\label{eq:proof_layer_group_1}
    O = (O_1 + O_2 + ... + O_n) W^T,
\end{equation}
only if $R_1 = R_2 = ... = R_n = R$.
\end{proof}

\begin{algorithm}
\caption{ TileTrans($\mathbf{M}$, $\mathcal{E}$) }\label{alg:tile_recon}
\begin{algorithmic}[1]
\Require DNN model $\mathbf{M}$ and metrics for importance score $\mathcal{E}$.
\Ensure Reparameterized DNN model $\mathbf{M}$.

\State $\mathbb{G} \gets \emptyset$ \Comment{Initialize the layer group}
\ForEach{$L \in \mathbf{M}$ }
    \State $\mathbf{G} \gets $ the parents of $L$
    \State $\mathbb{G} \gets \mathbb{G} \cup \mathbf{G}$
\EndFor
\State $\hat{\mathbb{G}} \gets \{\mathbf{G}~|~ \mathbf{G} \gets \mathbf{G}_1 \cup \mathbf{G}_2 \text{ if } \mathbf{G}_1 \cap \mathbf{G}_2 \neq \emptyset, \text{ and } \mathbf{G}_1,\mathbf{G}_2 \in \mathbb{G} \}$ \Comment{Merge the intersecting groups}
\ForEach{$\mathbf{G} \in \hat{\mathbb{G}}$ }
    \State $R \gets$ BuildTrans($\mathbf{G}, \mathcal{E}$) \Comment{Build the transformation}
    \ForEach{$L \in \mathbf{L}$ }
        \State Transform the weights of $L$ by $R$
    \EndFor
    \State Transform the weights in the children of the layers in $\mathbf{G}$ by $R^{-1}$
\EndFor

\end{algorithmic}
\end{algorithm}

According to Theorem~\ref{theorm:layer_group}, we design an algorithm for TileTrans,  which is shown in Algorithm~\ref{alg:tile_recon}.
From lines 2 to 4, we first find the sets of weights that must share the same transformation. 
Notably, some subsets in $\mathbb{G}$ may intersect each other.
For example, given two sets of layers $\mathbf{G}_0 \text{ and } \mathbf{G}_1$, $\mathbf{G_0} = \{L_A, L_B\}$ and $\mathbf{G_1} = \{L_B, L_C\}$, all $L_A$, $ L_B$, and $L_C$ must share the same transformation for their weights.
Therefore, we merge the subsets intersecting each other at line 6.
According to $\mathbf{G} \in \mathbb{G}$, we build the appropriate transformation $R$,  which is introduced later, for the weights at line 8.
To obtain the correct output, we transform the weights in the children of the layers in $\mathbf{G}$ by $R^{-1}$ at line 12.

The modern DNN models improve their performance by introducing nonlinear layers, such as ReLU~\cite{agarap2018deep}, pooling~\cite{albawi2017understanding}, and normalization~\cite{ioffe2015batch}.
For such nonlinear layers, we restraint the transformation $R$ to only the permutation of rows in Algorithm~\ref{alg:tile_recon}.
This restraint guarantees the correctness of the model because permutation is a linear transformation that only changes the location of the elements.
By contrast, a linear transformation that changes the scale of elements does not have this guarantee because it fails to recover the output of the model.

Additionally, the permutation of columns works in the TileTrans, but Algorithm~\ref{alg:tile_recon} needs to be ``reversed.''
In Algorithm~\ref{alg:tile_recon}, we define the permutation of rows $R$ for the parents at line 8 and permute the weight matrices at line 10. 
After that, we use the permutation of columns $R^{-1}$ to recover the output for the children at line 12.
Therefore, TileTrans in rows is a procedure from parents to children.
Meanwhile, TileTrans in columns is a procedure from children to parents.
Thus, for the permutation of columns, we build the transformation for children and then recover the output by the calculations in parents.

\begin{figure*}[t]
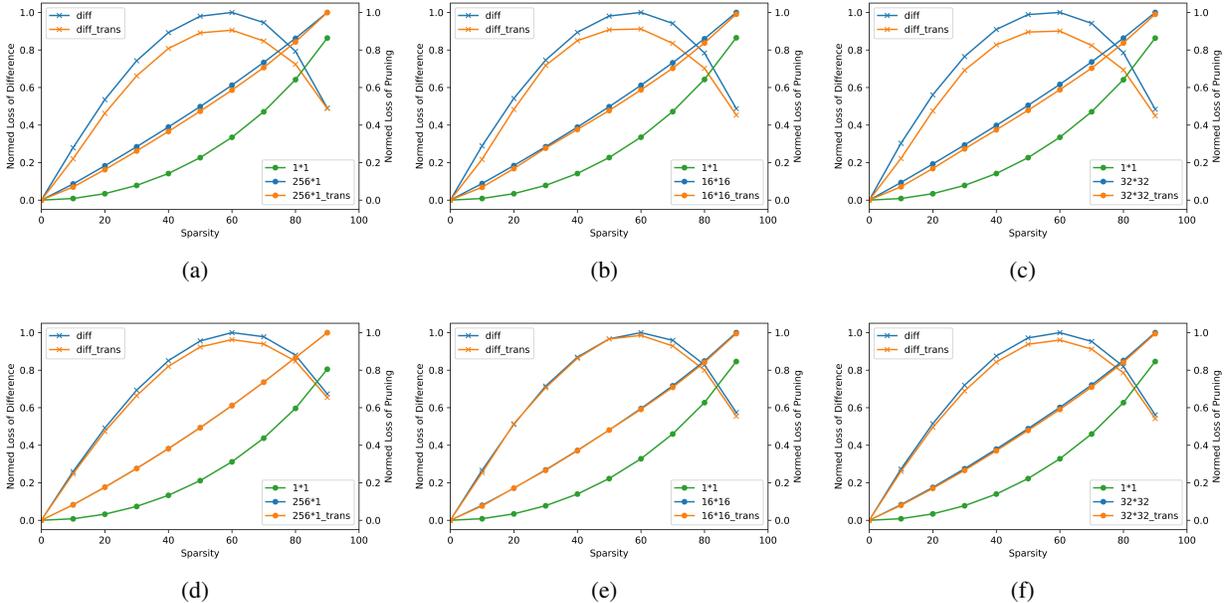

\centering
\begin{subfigure}{.3\textwidth}
  \centering
  \includegraphics[width=\linewidth]{fig/exp/alexnet_loss_256_1.pdf}
  \caption{}
  \label{fig:alex_loss_256_1}
\end{subfigure}%
\begin{subfigure}{.3\textwidth}
  \centering
  \includegraphics[width=\linewidth]{fig/exp/alexnet_loss_16_16.pdf}
  \caption{}
  \label{fig:alex_loss_16_16}
\end{subfigure}
\begin{subfigure}{.3\textwidth}
  \centering
  \includegraphics[width=\linewidth]{fig/exp/alexnet_loss_32_32.pdf}
  \caption{}
  \label{fig:alex_loss_32_32}
\end{subfigure}

\vskip\baselineskip
\vspace{-10pt}

\begin{subfigure}{.3\textwidth}
    \centering
    \includegraphics[width=\linewidth]{fig/exp/resnet_loss_256_1.pdf}
    \caption{}
    \label{fig:resnet_loss_256_1}
\end{subfigure}%
\begin{subfigure}{.3\textwidth}
    \centering
    \includegraphics[width=\linewidth]{fig/exp/resnet_loss_16_16.pdf}
    \caption{}
    \label{fig:resnet_loss_16_16}
\end{subfigure}
\begin{subfigure}{.3\textwidth}
    \centering
    \includegraphics[width=\linewidth]{fig/exp/resnet_loss_32_32.pdf}
    \caption{}
    \label{fig:resnet_loss_32_32}
\end{subfigure}
\caption{Loss of tile pruning on AlexNet (a, b, and c) and ResNet-34 (d, e, and f) that are pre-trained on the ImageNet for different sparsity. The tile shape in Figs.~(a) and (d) are  $256\times 1$, that in Figs.~(b) and (e) are $16\times 16$, and that in Figs.~(c) and (f) are $32 \times 32$. Specially, we define the result for unstructured tile pruning, i.e., the $1 \times 1$ tile pruning, as the baseline for tile pruning. Besides, the legend with ``trans'' denotes the result for the model reparameterized by TileTrans. Meanwhile, ``diff'' denotes the difference between tile pruning and unstructured pruning, which is defined as Eq.~(\ref{eq:recon_optimazation}). For simplicity, we normalize the pruning loss and difference by 0-1 normalization.}
\label{fig:model_loss}
\end{figure*}

\subsection{ Heuristic Permutation }
\label{sec:recon_selection}

Ideally, if all-important elements were aggregated in the same tile, all important elements then exist after tile pruning.
That is, the result of tile pruning becomes the same as that of unstructured pruning.
TileTrans aims to make the result of tile pruning as close as possible to that of unstructured pruning.
Because the loss for unstructured pruning is the smallest, we assume the loss for unstructured pruning is the baseline.
Based on this baseline, we give the optimization problem for the TileTrans as follows:
\begin{equation}\label{eq:recon_optimazation}
\begin{aligned}
    & \min \left\vert \sum_w^{w \in \mathit{f}_u(\mathbf{W}, s; \mathcal{E})} \mathcal{E}(w) - \sum_w^{w \in \mathit{f}_t(\mathcal{R}(\mathbf{W}), s; \mathcal{E})} \mathcal{E}(w) \right\vert,\\
    & \textrm{s.t.} \mathbf{W} \in \mathbb{W}, 0 \leq s \leq 1, 
\end{aligned}
\end{equation}
where $\mathit{f}_u: \mathbb{W} \to \mathbb{W}$ denotes the unstructured pruning method, $\mathit{f}_t: \mathbb{W} \to \mathbb{W}$ denotes the tile pruning method, and $\mathcal{R}:\mathbb{W} \to \mathbb{W} $ denotes the reparameterization method.
According to Eq.~(\ref{eq:recon_optimazation}), given specific sparsity $s$, we need to find the permutation that minimizes the difference between tile pruning and unstructured pruning methods.

However, searching for the optimal solution for Eq.~(\ref{eq:recon_optimazation}) is time consuming.
The most intuitive method to solve Eq.~(\ref{eq:recon_optimazation}) is to try all possible permutations and verify the effect, which costs too much time.
Given a model with $k$ layers and each layer with weights of $n$ rows, the time complexity to traverse all possible permutations is $O((n!)^k)$.
Therefore, we propose a heuristic method to build a relatively good permutation without searching.

\begin{algorithm}
\caption{ BuildTrans($\mathbf{G}$, $\mathcal{E}$) }\label{alg:sort_weights}
\begin{algorithmic}[1]
\Require Group of layers $\mathbf{G}$ and metrics for importance score $\mathcal{E}$.
\Ensure Permutation $R$.
\State $\mathbf{W} \gets$ all weights of the layers in $\mathbf{G} $
\State $\hat{W} \gets$  concatenate $ W \in \mathbf{W}$
\State $\hat{W}_R \gets \hat{W}$ 
\State Sort the rows in $\hat{W}_R$  \Comment{According to $\mathcal{E}$}
\State Build $R$ such that $ R \hat{W} = \hat{W}_R$
\end{algorithmic}
\end{algorithm}

We sort the rows according to their average importance scores, which are shown in Algorithm~\ref{alg:sort_weights}.
In line 1, we concatenate the weights that share the same transformation in the dimension of rows to evaluate them together. 
These weights can be concatenated because they have the same number of rows. 
After that, we sort the rows of the concatenated weights according to the average importance score, which is calculated by the metrics $\mathcal{E}$.
Finally, we build the transformation $R$ such that $R\hat{W} = \hat{W}_R$, where $\hat{W}$ denotes the concatenated weights, and $\hat{W}_R$ denotes the sorted concatenated weights.

The sorting procedure is based on two assumptions:
\begin{assumption}\label{assumption:normal_distribution}
The averages of the rows follow the normal distribution.
\end{assumption}
\begin{assumption}\label{assumption:variance}
The variance of importance scores for elements in a row is small.
\end{assumption}
\noindent
We make the assumptions depending on previous studies and our experience. 
Assumption~\ref{assumption:normal_distribution} is inspired by previous studies on channel pruning~\cite{luo2017thinet}, which prunes the elements in rows.
In channel pruning, the importance score for a row is evaluated by the average of the importance scores for the elements in the row.
After channel pruning, rows are deleted without a few important rows, meaning that a few important rows are in the matrix, whereas the other rows are unimportant.
Similarly, there are few large values in the normal distribution, while the others are smaller.
We use the normal distribution to approximate the distribution of the importance scores for the rows.
For Assumption~\ref{assumption:variance}, we make it empirically because we observed that importance scores for the elements in a row are usually all large when the row is important.
The unimportant and important elements are naturally separated after sorting rows depending on the assumptions.
Finally, the important elements are preserved after tile pruning because the important elements are in the same tiles.
In other words, we reduce the loss of tile pruning.
\begin{figure*}[t]
\centering
\begin{subfigure}{.32\textwidth}
  \centering
  \includegraphics[width=\linewidth]{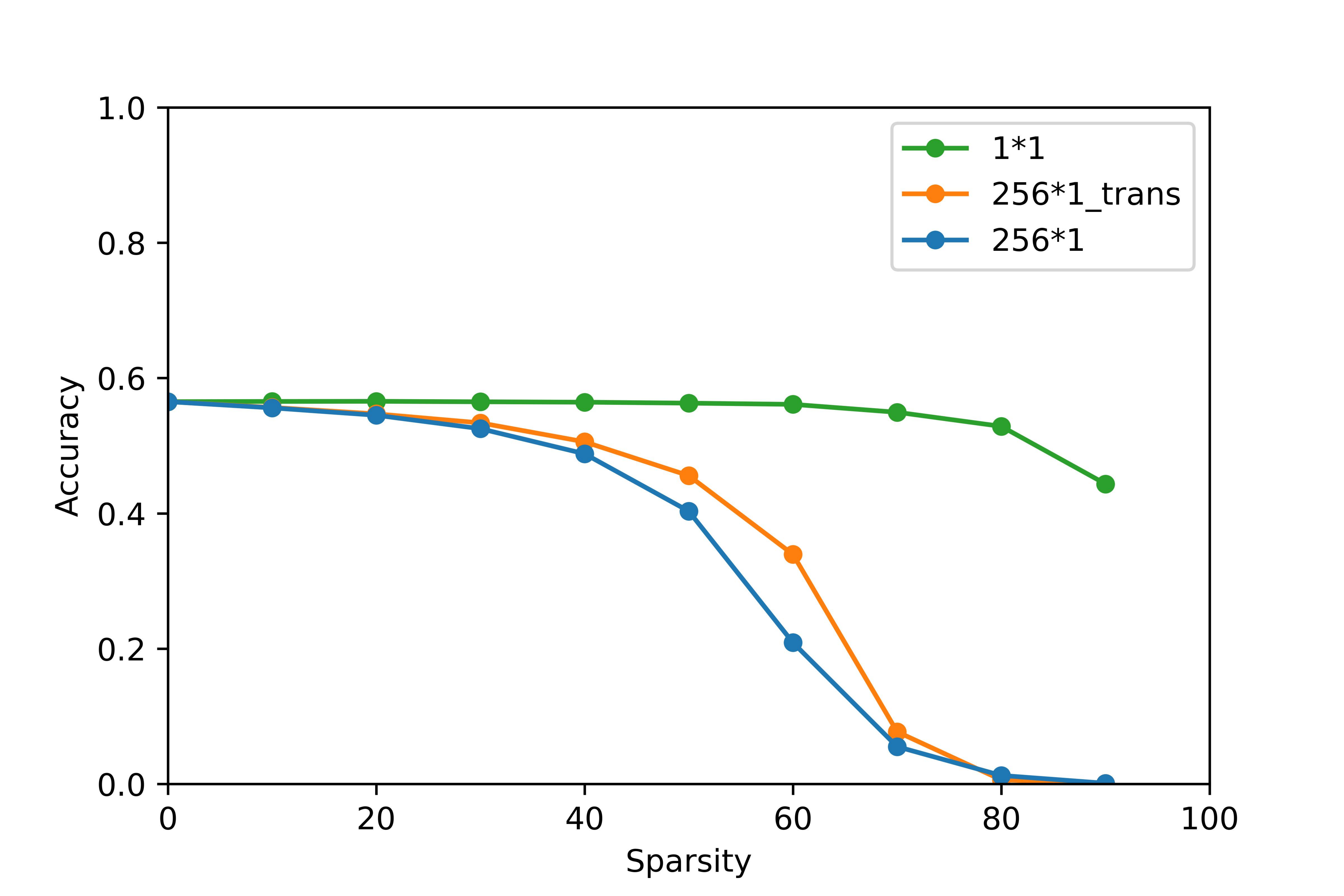}
  \caption{}
  \label{fig:alex_acc_256_1}
\end{subfigure}%
\begin{subfigure}{.32\textwidth}
  \centering
  \includegraphics[width=\linewidth]{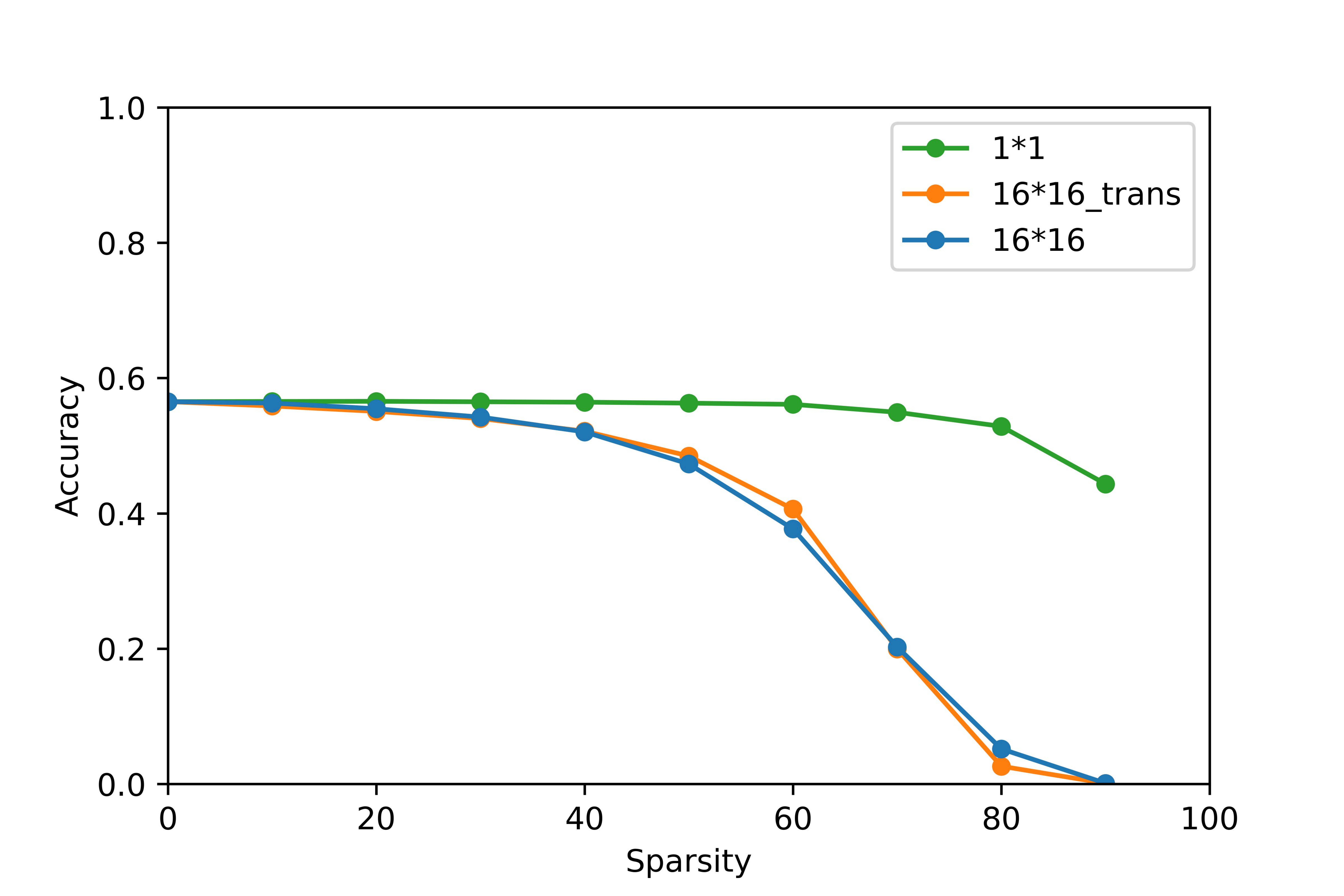}
  \caption{}
  \label{fig:alex_acc_16_16}
\end{subfigure}
\begin{subfigure}{.32\textwidth}
  \centering
  \includegraphics[width=\linewidth]{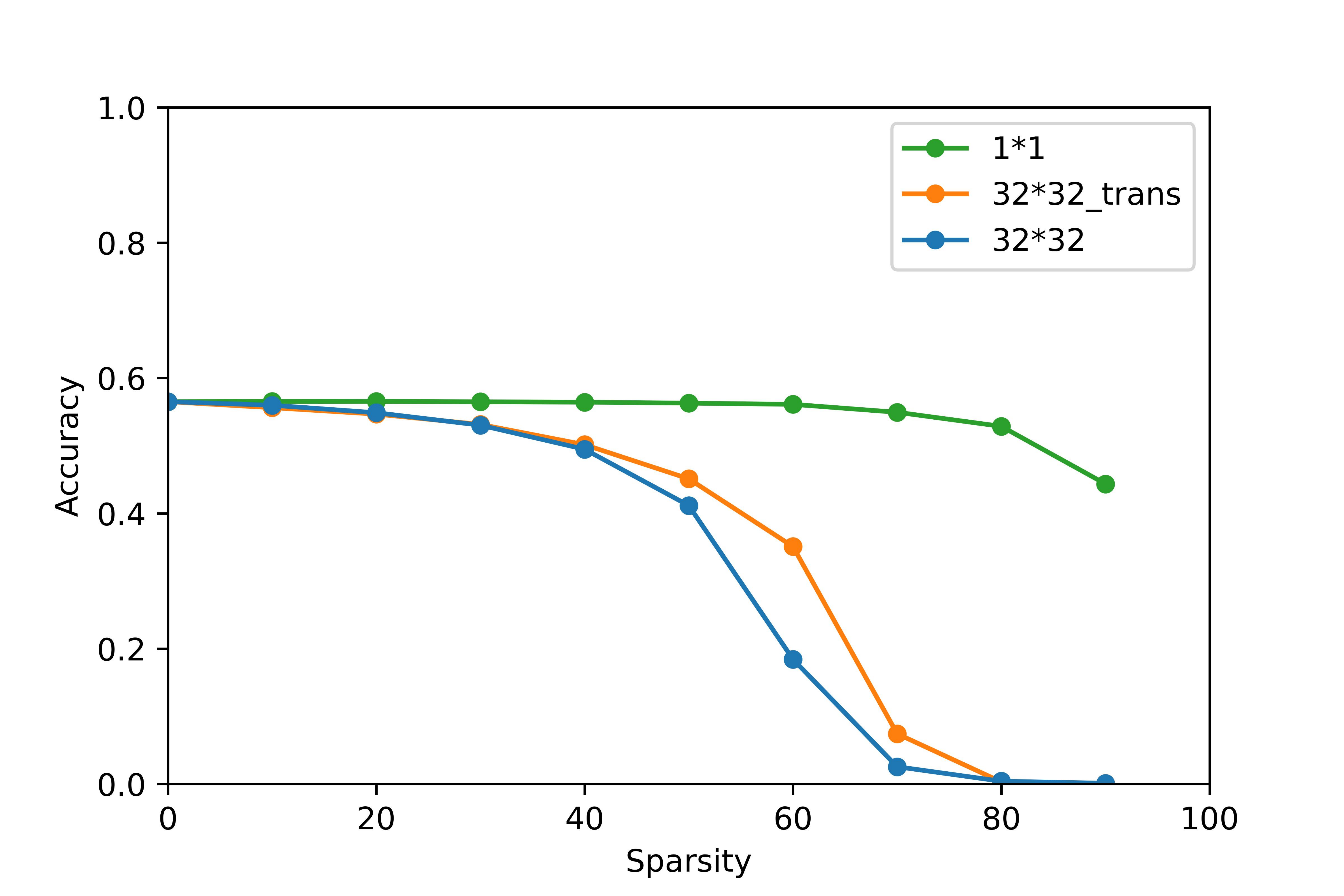}
  \caption{}
  \label{fig:alex_acc_32_32}
\end{subfigure}

\vskip\baselineskip

\begin{subfigure}{.32\textwidth}
    \centering
    \includegraphics[width=\linewidth]{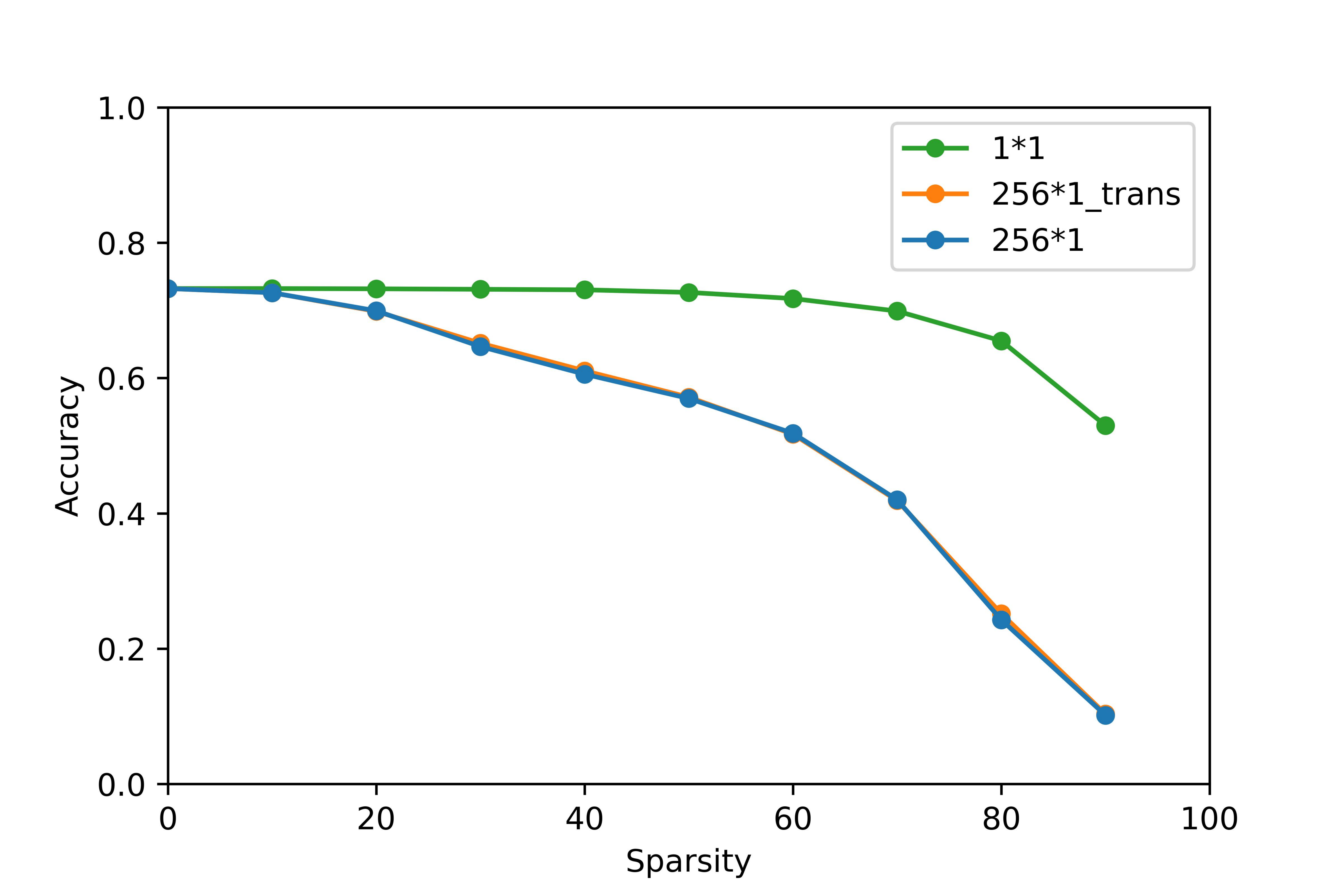}
    \caption{}
    \label{fig:resnet_acc_256_1}
\end{subfigure}%
\begin{subfigure}{.32\textwidth}
    \centering
    \includegraphics[width=\linewidth]{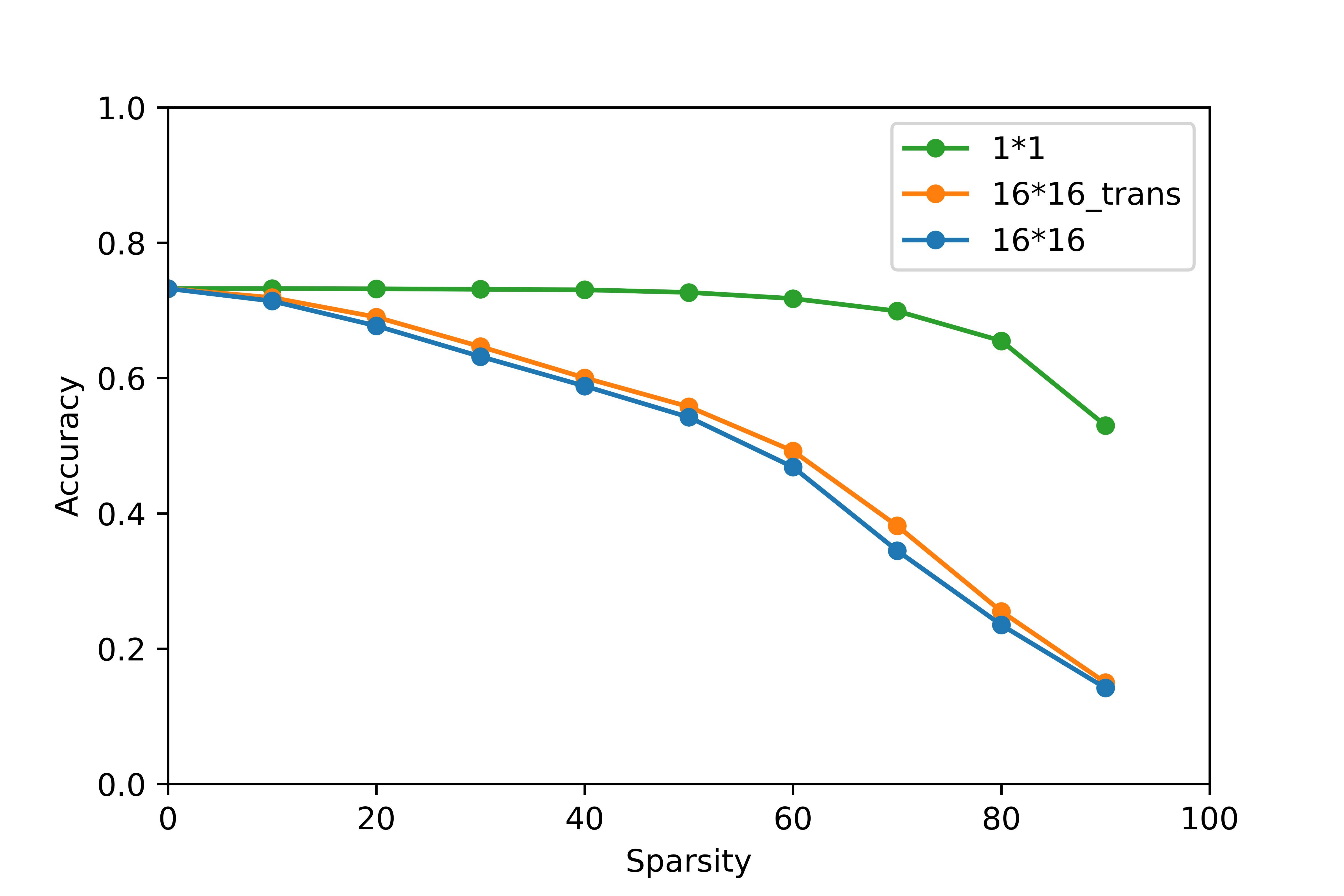}
    \caption{}
    \label{fig:resnet_acc_16_16}
\end{subfigure}
\begin{subfigure}{.32\textwidth}
    \centering
    \includegraphics[width=\linewidth]{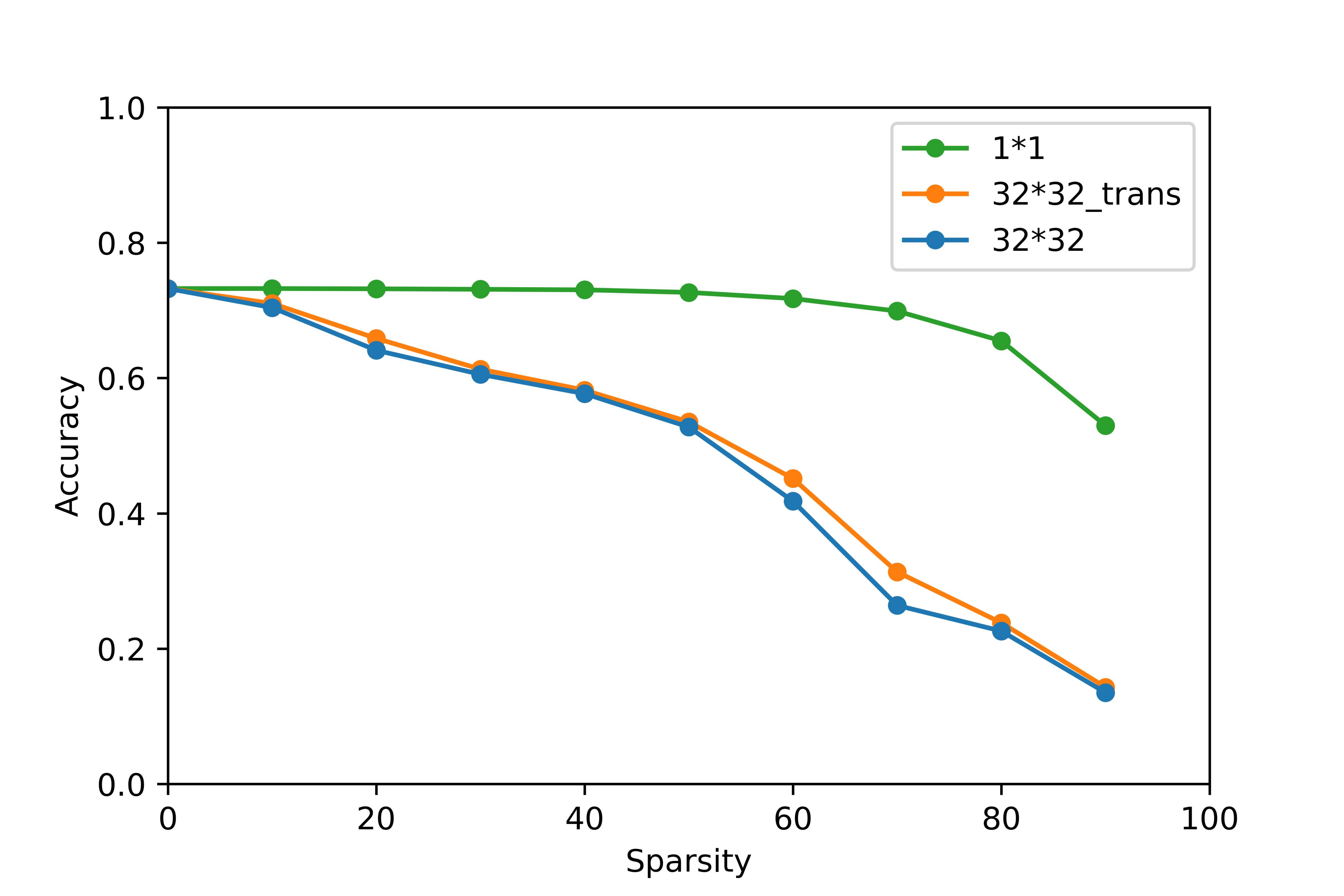}
    \caption{}
    \label{fig:resnet_acc_32_32}
\end{subfigure}
\caption{Accuracy of the pruned AlexNet (a, b, and c) and ResNet-34 (d, e, and f) that are pre-trained on the ImageNet for different sparsity. The tile shape in Figs.~(a) and (d) are $256\times 1$, that in Figs.~(b) and (e) are $16\times 16$, and that in Figs.~(c) and (f) are $32 \times 32$. Specially, we define the result for unstructured tile pruning, i.e., the $1 \times 1$ tile pruning, as the baseline for tile pruning. Besides, the legend with ``trans'' denotes the result for the model reparameterized by TileTrans.}
\label{fig:model_acc}
\end{figure*}

\section{Experiment}
\label{sec:experiment}

In the experiments, we evaluated the proposed TileTrans on the following metrics:
\begin{itemize}
    \item \textbf{Loss}. Pruning loss is our direct optimization objective. Therefore, we first show the pruning loss before and after the reparameterization of TileTrans. At the same time, we calculate the difference between tile pruning and unstructured pruning for the following two reasons. First, the difference in the pruning loss before and after TileTrans is negligible; therefore, we illustrate the normalized difference in figures for better observation. Second, we designed an optimization objective for TileTrans as Eq.~(\ref{eq:recon_optimazation}), which is the difference between tile pruning and unstructured pruning; thus, we calculate this difference to verify the heuristic permutation does reduce the pruning loss.
    \item \textbf{Accuracy}. The accuracy of the model is the most intuitive metric that evaluates the performance of the model. We are unaware of the performance of the pruned models only by evaluating pruning loss. Besides, the difference between the pruning loss with different tile shapes on the same model is negligible, so we cannot assess how tile shape affects the performance of TileTrans only by loss. Therefore, we calculate the accuracy of the pruned models to evaluate TileTrans further.
\end{itemize}

We design sets of controlled experiments on AlexNet and ResNet-34 to evaluate the performance of TileTrans. For one set, we first reparameterized the model by TileTrans and then pruned the model; for another set, we pruned the model directly. For simplicity, we choose the L1 norm as the criteria of importance score. Finally, we calculated the pruning loss and model accuracy of the pruned models.

Besides, we evaluate TileTrans on tile pruning of different tile shapes to assess how tile shape affects the performance of TileTrans. Expressly, we set the tile shape as $256 \times 1$, $16 \times 16$, or $32 \times 32$. For $256 \times 1$ and $16 \times 16$, the sizes are the same, while the aspect ratios are different; for $16 \times 16$ and $32 \times 32$, the aspect ratios are the same, while the sizes are different. Therefore, we know how the size and aspect ratio of tile affect the performance of TileTrans by comparing the result of these tile shapes. 

We realized TileTran with PyTorch and retrained the models on eight GeForce GTX 1080 GPUs. Our code is available on \url{https://github.com/LoonLi/TileTrans}.

\subsection{Loss}

The results of pruning loss are shown in Fig.~\ref{fig:model_loss}. Specially, we also illustrate the result of unstructured pruning, i.e., the $1 \times 1$ pruning, in Fig.~\ref{fig:model_loss} as the baseline of tile pruning. Accordingly, with specific sparsity, the difference between tile pruning and unstructured pruning is the distance between the point of tile pruning and unstructured pruning.

TileTrans reduced most pruning loss at the sparsity of about 60; however, it reduced little loss when the sparsity was sufficiently small or large. This is due to the minority tiles with all unimportant or important elements, which exist naturally in weight matrices. These tiles with all unimportant elements are deleted when the sparsity is small. In contrast, those with all important elements are still preserved when the sparsity is large. Consequently, TileTrans avoids degrading the pruning loss when the sparsity is sufficiently small or large. Meanwhile, those tiles with important and unimportant elements are permuted by TileTrans; thus, the pruning loss is reduced when the sparsity is moderate.

TileTrans reduced less pruning loss on ResNet-34 than on AlexNet. This is attributed to the following reasons: more complex architecture and smaller weight matrices. First, the architecture of ResNet-34 is more complex than that of AlexNet because there are multiple connections of residual in ResNet-34. According to Theorem 1, layers with residuals must share the transformation with other layers that send the residual to them. In contrast, AlexNet is a simple architecture in which no layer needs to share the same transformation. Thus, searching for good transformations for ResNet-34 is more difficult. Second, the weight matrices of ResNet-34 are small for the permutation. Specifically, the heights of most weight matrices in ResNet-34 are less than 512. Meanwhile, there are large weight matrices with a height of 4096 in AlexNet. In other words, we have fewer choices to permute the weight matrices in rows on ResNet-34 than on AlexNet.

Pruning loss is the most direct metric to evaluate TileTrans, but it is not intuitive enough for understanding the performance of the models. Specifically, we are unaware of the difference between different tile shapes since their results are very similar in the same model. Therefore, we also evaluate TileTrans on the metric of the model accuracy.

\subsection{Accuracy}

We define accuracy as the probability that models correctly infer the types of images in the test data. The accuracy $P_a$ is calculated as $P_a = N_c / N_t$, where $N_t$ is the amount of test data $N_t$ and $N_c$ is the amount of correctly inferred images.

TileTrans improved the accuracies of both AlexNet and ResNet-34 after tile pruning, which is a natural result because TileTrans reduced the pruning loss. We show results of model accuracy in Fig.~\ref{fig:model_acc}. Specifically, the accuracies of both AlexNet and ResNet-34 were improved most under the tile shape of $32 \times 32$. In detail, the pruned AlexNet was improved by up to 17\% at the sparsity of 60\%; the pruned ResNet-34 was improved by up to 5\% at the sparsity of 70\%. Specially, the accuracy of the pruned AlexNet decreased after TileTrans at the sparsity of 80\% under the tile shape of $16 \times 16$. However, in this case, the pruning loss was reduced according to Fig.~\ref{fig:alex_loss_16_16}. Consequently, this decrease is not caused by increased pruning loss and may be improved by changing the criteria of importance score. 

Between the tile shapes of $256 \times 1$ and $16 \times 16$, TileTrans improved the accuracy more for $256 \times 1$ on AlexNet while for $16 \times 16$ on ResNet-34. According to this result, we draw two conclusions about the performance of TileTrans. On the one hand, TileTrans improve more accuracy for strip-shaped tile when the weight matrix is large enough. On the other hand, TileTrans improves more accuracy for square-shaped tiles when the weight matrix is too small. ResNet-34 consists of multiple convolutional layers with heights smaller than $512$. As a result, TileTrans reduced little pruning loss on ResNet-34 for $256 \times 1$ because the permutation cannot exchange the elements among tiles. Meanwhile, AlexNet contains two large linear layers whose weight matrices are $4096 \times 4096$. Therefore, the important and unimportant elements were combined in different tiles on AlexNet for $256 \times 1$. In this case, the unimportant elements are deleted more precisely for $256 \times 1$ than $16 \times 16$.

Between the tile shapes of $16 \times 16$ and $32 \times 32$, TileTrans improved the accuracy more for $32 \times 32$ on both models. Consequently, we conclude that TileTrans improves accuracy more for larger tile sizes. The probability of a tile in the weight matrix filled with important elements is low when the tile size is large. Meanwhile, TileTrans builds weight matrices with more tiles that are filled with important elements. In other words, TileTrans reduces pruning loss more when the tile size is large.

In summary, we conclude three rules for TileTrans as follows:
\begin{enumerate}
    \item TileTrans reduces more loss for strip-shaped than square-shaped tiles when the sizes of tiles are the same, and the weight matrix is large enough for permutation.
    \item TileTrans reduces more loss for square-shaped than strip-shaped tiles when the sizes of tiles are the same, and the weight matrix is small.
    \item TileTrans reduces more loss for the tiles with large size.
\end{enumerate}

\section{Conclusion}
\label{sec:conclusion}

In this study, we proposed a one-shot reparameterization method, called TileTrans, to reduce the loss of tile pruning on DNNs by permuting the weight elements.
We heuristically build an appropriate permutation to avoid time-consuming searching for the optimal permutation.
Besides, TileTrans requires no retraining and keeps the model architecture unchanged. 
The experiments demonstrated that TileTrans reduced the loss of tile pruning on AlexNet and ResNet-34.

\section*{Acknowledgments}
This research was in part supported by ``Program for Leading Graduate
Schools'' of the Ministry of Education, Culture, Sports, Science and
Technology, Japan and the Japan Society for the Promotion of Science
KAKENHI grant number 20K21794.


\end{document}